\RequirePackage{etoolbox}
\patchcmd{\bibliographystyle}{#1}{abbrvnat}{}{}

\documentclass[11pt]{article}

\usepackage[utf8]{inputenc}
\usepackage[T1]{fontenc}
\usepackage{authblk}

\usepackage[a4paper,margin=1in]{geometry}
\usepackage{amsmath,amssymb,amsthm}
\usepackage{graphicx}
\usepackage{tikz}
\usetikzlibrary{arrows.meta,positioning,shapes.geometric}

\usepackage{booktabs}
\usepackage{algorithm2e}
\usepackage{rotating}
\usepackage{subcaption}

\usepackage[numbers,sort&compress]{natbib}

\usepackage[hidelinks]{hyperref}

\DeclareMathOperator{\Var}{Var}
\DeclareMathOperator{\Cov}{Cov}

\newtheorem{theorem}{Theorem}
\newtheorem{proposition}[theorem]{Proposition}

\newtheorem{remark}[theorem]{Remark}

\theoremstyle{definition}

\title{
How Does Bayesian Causal Discovery Fail?\\
\large Characterising Structural Consequences in Linear Gaussian Networks
under Latent Confounding
}

\author[1]{Debargha Ghosh}
\author[1]{Silja Renooij}
\author[2]{Anna V. Kononova}

\affil[1]{Department of Information and Computing Sciences, Utrecht University}
\affil[2]{Leiden Institute of Advanced Computer Science, Leiden University}

\date{}

\begin{document}

\maketitle

\begin{abstract}
Bayesian causal discovery is widely used for its ability to quantify epistemic uncertainty over directed acyclic graphs (DAGs) through posterior inference. However, its behaviour under latent confounding remains poorly understood, as existing work typically notes that confounding breaks identifiability without characterising how the posterior distribution over DAGs responds. In this work, we analyse posterior behaviour under latent confounding in linear Gaussian causal models, focusing on additive latent confounding between exactly two observed variables. We derive a critical correlation threshold above which the score function
favours graphs with a  spurious edge between the confounded variables, and show that
this threshold decreases with sample size -- more data lowers the
correlation required for the spurious edge to be favoured. Beyond this threshold, we characterize two distinct posterior failure regimes determined by the local structure around the confounded variables. Our findings are supported by exact posterior computations on multiple graph structures, demonstrating both the predicted failure regimes.

\end{abstract}

\section{Introduction}
 
Discovering the causal relationships underlying a system of variables is a central problem
across scientific domains. Bayesian networks and structural equation models provide a principled framework for representing such dependencies \citep{pearl2009causality,koller2009probabilistic}.
Structure learning methods aim to recover a graph representing the underlying data generating process from observational data, but are inherently limited by statistical uncertainty and identifiability constraints \citep{glymour2019review}. In contrast, Bayesian causal structure learning
infers a posterior distribution $p(G \mid D)$ over candidate directed acyclic graphs (DAGs) representing the underlying causal system
\citep{friedman2003being}, enabling principled quantification of epistemic uncertainty. This
distribution allows assessing confidence in inferred structures and supports downstream tasks
such as experimental design and active causal discovery \citep{tong2001active,murphy2001active}, where uncertainty over candidate
graphs plays a central role. However, how this posterior uncertainty changes under violations
of modelling assumptions, such as latent confounding, is not well understood.

The reliability of posterior-based causal conclusions from observational data---namely,
whether posterior probabilities faithfully reflect uncertainty over plausible structures---depends
on two key components: identifiability and scoring. On the identifiability side, under the assumptions of the causal Markov
condition, causal sufficiency (lack of latent confounding), and
faithfulness, observational data can distinguish causal structures only
up to Markov equivalence \citep{verma1990equivalence,vonk2023disentangling}, even under infinite data. On the scoring side, the choice of score---such as the Bayesian Gaussian equivalent (BGe) score in linear Gaussian models \citep{geiger2002parameter,kuipers2014addendum}—determines the likelihood and prior that govern how evidence is translated into posterior probabilities over candidate
causal structures, under assumptions such as independent exogenous noise and Gaussianity. While both identifiability and scoring are well understood under standard assumptions, it
remains unclear how their interplay affects posterior behaviour when these assumptions are
violated. 

Of these assumptions, causal sufficiency is particularly consequential for causal discovery from observational data. The assumption is not testable in practice, yet it is routinely violated in many applications. Causal discovery methods are widely applied to settings such as gene regulatory network inference from expression data \citep{friedman2000using, sachs2005causal, affeldt2016learning}. Here, unmeasured cell states, batch effects, and latent regulatory factors can induce confounding effects~\citep{leek2012sva, chen2017controlling, wang2017efficient}. Similarly, in clinical epidemiology, Bayesian structure learning is used on observational patient data where unmeasured characteristics are a recognised source of confounding bias \citep{kyrimi2021bayesian}. For Bayesian causal discovery, these practical
violations raise the question of how posterior uncertainty behaves under latent confounding.

Existing work on Bayesian causal discovery and latent confounding falls into two main categories. \textbf{(i)} A substantial line of work develops Bayesian inference over DAGs under the assumption of causal sufficiency---including order MCMC \citep{friedman2003being} and GADGET \citep{viinikka2020gadget}. \textbf{(ii)} A parallel line addresses confounding by switching to richer model classes, motivated by the fact that in the presence of latent confounders the causal structure over observed variables is generally not identifiable as a causally sufficient DAG from observational data \citep{spirtes2000causation, richardson2002ancestral}. This includes constraint-based methods such as FCI and its variants \citep{spirtes2000causation, zhang2008fci}, which return partial ancestral graphs, and Bayesian inference jointly over acyclic directed mixed graphs (ADMGs) and latent variables \citep{ashman2023neural}---solving the problem by changing the inference target. Empirical benchmarks under assumption violations further show that most benchmarked causal discovery methods do not explicitly support confounding effects, and that spurious correlations induced by latent confounding degrade edge selection and graph recovery \citep{montagna2023assumption}. None of these approaches address the \textit{critical question} of what happens when Bayesian causal discovery methods operating on DAGs are applied to data with unobserved confounding: whether the posterior appropriately reflects the resulting uncertainty, or instead concentrates on systematically incorrect inferred structures. More broadly, is this behaviour structured and predictable, or inherently unreliable?

In this paper, we address the above question by providing a structural characterisation of posterior behaviour under latent confounding. We consider a simple but representative setting of a single additive latent confounder affecting exactly two observed variables in linear Gaussian models with additive noise. Within this setting, we derive a sample-size-dependent critical correlation threshold, beyond which the posterior begins to shift towards incorrect causal graphs.
We then identify two qualitatively distinct failure regimes in which posterior behaviour is governed by the structural properties of the local region affected by confounding.


\section{Background}
\label{sec:background}

\subsection{Linear Gaussian SCMs and Latent Confounding}
\label{sec:scm}

A \emph{structural causal model} (SCM) over variables
\(\mathbf{X} = (X_1, \dots, X_n)\) specifies each variable as a deterministic
function of its causal parents and an independent exogenous noise term
\citep{pearl2009causality}. We work with \emph{linear Gaussian} SCMs, in
which each variable follows
\[
    X_i = \sum_{j \in \mathrm{pa}(i)} \beta_{ij} X_j + \varepsilon_i,
    \qquad \varepsilon_i \sim \mathcal{N}(0, \sigma_i^2),
\]
where \(\mathrm{pa}(i)\) is the parent set of \(X_i\) in a directed acyclic
graph (DAG) \(G\), the noise terms \( \varepsilon_i\) are mutually independent, and the
coefficients \(\beta_{ij}\) are bounded away from zero to ensure
non-trivial dependence along each edge. The DAG, coefficients, and noise
variances jointly induce a centered Gaussian distribution over
\(\mathbf{X}\).

The setting above assumes that every common cause of any pair of
observable variables is itself observable. When this fails, an unobserved
variable \(H\) may influence multiple observable variables simultaneously.
We adopt a Gaussian \emph{additive latent confounder} model: if
\(H \sim \mathcal{N}(0, \sigma_H^2)\) confounds two observable variables
\(X_j\) and \(X_k\) that are non-adjacent in the underlying DAG, their
structural equations become
\begin{equation}
    X_j = \sum_{\ell \in \mathrm{pa}(j)} \beta_{j\ell} X_\ell + \gamma_j H + \varepsilon_j,
    \quad
    X_k = \sum_{\ell \in \mathrm{pa}(k)} \beta_{k\ell} X_\ell + \gamma_k H + \varepsilon_k,
    \label{eq:confounder}
\end{equation}
with confounding coefficients \(\gamma_j, \gamma_k\) and noise terms
independent of \(H\). We choose \(H\) Gaussian so that the joint
distribution over \(\mathbf{X}\) remains Gaussian after marginalising
\(H\). Marginalisation introduces an additional covariance contribution
\(\gamma_j \gamma_k \sigma_H^2\) between \(X_j\) and \(X_k\): since \(X_j\) and \(X_k\) are not
connected by an edge, no DAG representing the observable causal structure
can account for this dependence.

\subsection{Bayesian Network Structure Learning}
\label{sec:bsl}

A Bayesian network represents the dependencies among variables
\(X_1, \dots, X_n\) through a directed acyclic graph (DAG) \(G\), in which
each variable is conditionally independent of its non-descendants given
its parents \(\mathrm{pa}_G(i)\) \citep{pearl2009causality}. The joint
distribution factorises as
\begin{equation}
    p(X_1, \dots, X_n \mid G) = \prod_{i=1}^n p\bigl(X_i \mid \mathrm{pa}_G(i)\bigr),
    \label{eq:factorisation}
\end{equation}
and different DAGs encode different sets of conditional independence
relations.

The goal of Bayesian structure learning is to infer the DAG from data.
Since different DAGs impose different conditional independence relations
and therefore different factorisations of the joint distribution, the
marginal likelihood \(P(\mathcal{D} \mid G)\) of the observed data depends
on the specific DAG \(G\). We assume the dataset consists of \(N\)
observations \(\mathcal{D} = \{\mathbf{x}^{(1)}, \dots, \mathbf{x}^{(N)}\}\),
where each \(\mathbf{x}^{(j)} \in \mathbb{R}^n\) is a realisation of the
random vector \(\mathbf{X} = (X_1, \dots, X_n)^\top\). Following the
Bayesian paradigm \citep{friedman2003being},
the posterior over DAGs is
\begin{equation}
    P(G \mid \mathcal{D})
    \;=\;
    \frac{P(\mathcal{D} \mid G)\, P(G)}{P(\mathcal{D})}
    \;\propto\;
    P(\mathcal{D} \mid G)\, P(G),
    \label{eq:posterior}
\end{equation}
where \(P(G)\) is the prior over DAGs and
\(P(\mathcal{D}) = \sum_{G'} P(\mathcal{D} \mid G')\, P(G')\) is a
normalising constant independent of \(G\). In the absence of prior
knowledge over graphs, we assume all DAGs to be equally likely a priori,
\(P(G) = c\) for all \(G\), so the posterior is proportional to the
marginal likelihood.

The marginal likelihood is obtained by integrating out the local
parameters \(\theta\)---the edge coefficients \(\{\beta_{ij}\}\) and noise
variances \(\{\sigma_i^2\}\) defining the conditional distributions of
each node given its parents---against a parameter prior:
\[
    P(\mathcal{D} \mid G)
    \;=\;
    \int P(\mathcal{D} \mid G, \theta)\, P(\theta \mid G)\, d\theta.
\]

Two DAGs are \emph{Markov equivalent} if they imply the same conditional
independence relations among \(X_1, \dots, X_n\); observational data
alone cannot distinguish between them \citep{verma1990equivalence,
chickering2002optimal}. By definition, two DAGs are
Markov equivalent if and only if they share the same skeleton (the
undirected graph obtained by ignoring edge directions) and the same set
of \emph{colliders}: triples \(X_i \to X_j \leftarrow X_k\) in which
\(X_i\) and \(X_k\) are not adjacent. The resulting equivalence classes
are called Markov equivalence classes (MECs). A marginal likelihood is
\emph{score-equivalent} if it assigns identical values to all DAGs within
an MEC, since these DAGs are
observationally indistinguishable. There exist only two
\textbf{Bayesian} scores that satisfy this property: the discrete BDe score for
categorical data and the Gaussian BGe score for linear Gaussian SCMs
\citep{grzegorczyk2023being}. We focus on
the BGe score throughout this work.

For linear Gaussian SCMs, the joint distribution of \(\mathbf{X}\) under
a DAG \(G\) is multivariate Gaussian,
\(\mathbf{X} \mid G \sim \mathcal{N}_n(\boldsymbol{\mu}^G, \Sigma^G)\),
with mean vector \(\boldsymbol{\mu}^G\) and covariance \(\Sigma^G\) chosen
to be coherent with the factorisation~\eqref{eq:factorisation} implied by
\(G\). The BGe score \citep{geiger2002parameter}
places a fully conjugate normal--Wishart prior on the joint mean and
precision \((\boldsymbol{\mu},\Sigma^{-1})\),
\[
    \boldsymbol{\mu} \mid \Sigma \sim \mathcal{N}_n\bigl(\boldsymbol{\mu}_0, \alpha_\mu^{-1} \Sigma\bigr),
    \qquad
    \Sigma^{-1} \sim \mathcal{W}_n(\alpha_w, T),
\]
parameterised by a prior mean \(\boldsymbol{\mu}_0 \in \mathbb{R}^n\),
prior precision \(\alpha_\mu > 0\) on the mean, prior degrees of freedom
\(\alpha_w > n - 1\), and positive-definite scale matrix
\(T \in \mathbb{R}^{n \times n}\). Here, \(\mathcal{W}_n(\alpha_w, T)\) denotes the \(n\)-dimensional Wishart distribution. Conjugacy with the Gaussian likelihood yields a normal--Wishart posterior, 
and  the marginal likelihood \(P(\mathcal{D} \mid G)\) of G can be computed analytically and is called the BGe score of G. Due to modularity property of the score, the log of the score decomposes into a sum of local scores:
\begin{equation}
   s(G) := \log P(\mathcal{D} \mid G) = \log \prod_{i=1}^{n} P\bigl(\mathcal{D} \mid i, \operatorname{pa}_G(i)\bigr) = \sum_{i=1}^{n} s\bigl(i, \operatorname{pa}_G(i)\bigr).
    \label{eq:decomp}
\end{equation}
Since the logarithm is monotonic and we use a uniform prior over graphs,
comparisons of \(s(G)\) values are equivalent to comparisons of the scores; we
therefore work with \(s(G)\) throughout. A more detailed exposition of
the BGe score can be found in \citep{geiger2002parameter, kuipers2014addendum}.


\section{Spurious Edge Inclusion under Latent Confounding}

The latent confounder $H$ in Equation~\eqref{eq:confounder}
induces an additional dependence between $X_j$ and $X_k$. Once this dependence
is strong enough, it cannot be explained by any directed edge among the
observed variables. Consequently, the marginal likelihood $P(\mathcal{D} \mid G)$ in
Bayesian structure learning favours DAGs that include a directed edge
$X_k \to X_j$ (or its reverse) to account for this correlation. We call this induced edge as spurious edge in the context of causality as it represents an incorrect causal link.

In this section, we quantify \textit{how strong} this confounding induced correlation must be for DAGs with such an edge to
be favoured by the marginal likelihood. We begin by deriving a critical correlation threshold based
on the partial correlation induced by the confounder. We then show that
asymmetry in the parent sets makes one direction of an edge strictly easier to be induced
than the other.

\subsection{Critical Threshold for Spurious Edge Inclusion}
\label{sec:critical-threshold}

The key to our analysis is that comparing two DAGs that differ only in a
single edge reduces to a local computation. Latent confounding induces an
excess dependence between the confounded variables that the true graph
cannot explain; the relevant question is whether the score favours adding
a direct edge between them to absorb this dependence. We therefore compare
the graph without this edge against the one with it, and ask when the
latter is preferred. Recall that the log of the BGe score is given by Equation~\eqref{eq:decomp}.

Let \(G = (V, E)\) be the underlying causal graph over variables
\(\{X_1, \dots, X_n\}\), with latent confounding between \(X_k\) and
\(X_j\). Consider two DAGs \(G\) and \(G'\) that are identical except that
\(G'\) contains the additional edge \(X_k \to X_j\), i.e.,
\(\operatorname{pa}_{G'}(j) = \operatorname{pa}_G(j) \cup \{k\}\), while all
other parent sets are unchanged. By decomposability 
in~\eqref{eq:decomp}, every node-wise term is identical in the two
graphs except that of \(X_j\), so their score difference reduces to a
single local term:
\[
s(G') - s(G)
= s(j, \operatorname{pa}_G(j) \cup \{k\}) - s(j, \operatorname{pa}_G(j)).
\]
The score favours including the edge precisely when this difference is
positive. Deciding whether to include \(X_k\) as a parent of \(X_j\) therefore
reduces to a \emph{local score comparison} involving only node \(j\), its
current parents \(P := \operatorname{pa}_G(j)\), and the candidate parent
\(k\) -- independent of the rest of the graph. The following proposition formalises this local score comparison, expressing
it in terms of the partial correlation induced between \(X_j\) and \(X_k\)
given \(P\) given by \(\rho_{jk \mid P}\) and identifying
the critical value at which the score begins to favour the additional edge. See Appendix~\ref{app:threshold-proof} for a detailed derivation.

\begin{proposition}[Critical correlation threshold]
\label{prop:local-critical-threshold}
Assume a regular linear Gaussian data-generating process on a causal
graph \(G = (V, E)\) over variables \(\{X_1, \dots, X_n\}\), with latent
confounding between \(X_k\) and \(X_j\). Let \(P = \operatorname{pa}_G(j)\)
be the current parent set of \(X_j\), and let \(G'\) denote the graph
obtained from \(G\) by adding the edge \(X_k \to X_j\) (i.e.,
\(\operatorname{pa}_{G'}(j) = P \cup \{k\}\), with all other parent sets
unchanged). Then, for sufficiently large sample size \(N\) and fixed BGe
hyperparameters, the BGe score prefers \(G'\) over \(G\) whenever the
partial correlation \(\rho_{jk \mid P}\) exceeds a critical value
\(\rho_c\); that is, when
\[
    N \log\!\left(\frac{1}{1-\rho_{jk \mid P}^2}\right) > \log N .
\]
The critical partial correlation \(\rho_c\) is defined by the equality
\[
    N \log\!\left(\frac{1}{1-\rho_c^2}\right) = \log N ,\qquad\text{i.e.}\qquad
    \rho_c^2 = 1 - N^{-1/N}.
\]
so that the score favours the additional edge whenever
\(\rho_{jk \mid P} > \rho_c\).
\end{proposition}

\begin{remark}
The threshold condition has the closed-form solution
\(\rho_c^2 = 1 - e^{-\log N / N}\). For large \(N\),
expanding the exponential gives
\[
    \rho_c \approx \sqrt{\frac{\log N}{N}}.
\]
The graph-independent critical partial correlation therefore decreases with sample size:
more data lowers the correlation required to favour the edge, so weaker correlation \(\rho_{jk \mid P}\) caused by
confounding suffices to induce the spurious edge.
\end{remark}

\subsection{Asymmetry and the Preferred Direction}
\label{sec:asymmetry-preferred}

The above analysis considered adding \(X_k\) as a parent of \(X_j\),
conditioning on the set of observed parents \(P\) of \(X_j\). Without prior
causal knowledge, the reverse orientation \(X_j \to X_k\) is also a
candidate and the two compete score-wise. We show by an example that the
asymmetry of parent sets causes DAGs with a particular orientation of the spurious edge 
to be preferred over the DAGs with the reverse orientation of the same edge.

Consider the small SCM in Figure~\ref{fig:asymmetry-scm}, where \(X_j\)
has the observed parents set $P = \{ X_{p_1}, X_{p_2} \}$ . Let \(X_P = (X_{p_1}, X_{p_2})\) with coefficient
vector \(\beta_P\), \(X_k\) have no observed parents, and the latent
variable \(H\) confounds the pair:
\[
H \sim \mathcal{N}(0, \sigma_H^2), \quad
X_k = \gamma_k H + \varepsilon_k, \quad
X_j = \beta_P^\top X_P + \gamma_j H + \varepsilon_j,
\]
with mutually independent Gaussian noise terms
\(\varepsilon_j \sim \mathcal{N}(0, \sigma_j^2)\) and
\(\varepsilon_k \sim \mathcal{N}(0, \sigma_k^2)\), independent of \(H\).

The partial correlations for including the two candidate edges given the remaining graph stays the same are:
\[
\rho_{jk \mid P} = \frac{\gamma_j \gamma_k \sigma_H^2}
    {\sqrt{(\gamma_j^2 \sigma_H^2 + \sigma_j^2)(\gamma_k^2 \sigma_H^2 + \sigma_k^2)}},
\]
\[
\rho_{kj \mid \varnothing} = \frac{\gamma_j \gamma_k \sigma_H^2}
    {\sqrt{(\gamma_j^2 \sigma_H^2 + \sigma_j^2 + \beta_P^\top \Sigma_P \beta_P)(\gamma_k^2 \sigma_H^2 + \sigma_k^2)}},
\]
where \(\beta_P^\top \Sigma_P \beta_P = \operatorname{Var}(\beta_P^\top X_P)\)
is the variance of \(X_j\) explained by its observed parents. The two
expressions share the same numerator; the reverse edge's denominator
carries the extra \(\beta_P^\top \Sigma_P \beta_P\) term because the reverse
direction conditions on no parents. Hence \(\rho_{jk \mid P} > \rho_{kj \mid \varnothing}\).
So, whenever the
reverse partial correlation exceeds the critical threshold \(\rho_c\),
the forward one does too (by Proposition~\ref{prop:local-critical-threshold});
the converse is not generally true. In the regime where both correlations
exceed \(\rho_c\) and both candidate DAGs score higher than \(G\), the
choice between them reduces to a direct comparison of their local terms:
\[
s(G \cup \{X_k \to X_j\}) - s(G \cup \{X_j \to X_k\})
\approx \frac{N}{2}\left[
\log\!\left(\frac{1}{1-\rho_{jk \mid P}^2}\right) -
\log\!\left(\frac{1}{1-\rho_{kj \mid \varnothing}^2}\right)
\right] > 0,
\]
so \(G \cup \{X_k \to X_j\}\) is strictly preferred.
This asymmetry implies that, when a latent confounder induces a spurious
edge, the score favours the orientation directed \emph{towards} the
confounded variable with parents in the underlying causal graph over the orientation directed
towards the one without. More generally, the preferred orientation points
to the variable whose parents explain more of its variance: conditioning
on these parents reduces the residual variance and raises the induced
partial correlation, leading to DAGs with this orientation having a higher score than the one with opposite orientation.

\begin{figure}[h]
\centering
\begin{tikzpicture}[
    >=stealth, scale=0.7, every node/.style={font=\small}
]
\node (p1) {$X_{p_1}$};
\node (p2) [below of=p1] {$X_{p_2}$};
\node (xj) [right of=p1, xshift=1cm, yshift=-0.7cm] {$X_j$};
\node (h)  [above of=xj, yshift=-0.3cm, xshift=0.5cm] {$H$};
\node (xk) [right of=xj, xshift=1cm] {$X_k$};
\draw[->] (p1) -- (xj);
\draw[->] (p2) -- (xj);
\draw[->, dotted] (h) -- (xj);
\draw[->, dotted] (h) -- (xk);
\end{tikzpicture}
\caption{A minimal SCM illustrating the asymmetry: \(X_j\) has observed
parents \(P\), \(X_k\) has none, and the latent variable \(H\) confounds
the pair.}
\label{fig:asymmetry-scm}
\end{figure}
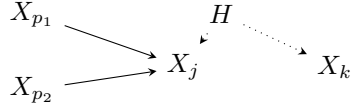


\section{Posterior Behaviour Beyond Critical Correlation}
\label{sec:structural}

We now examine \emph{how} Bayesian causal discovery behaves when the partial correlation induced by 
confounding exceeds the threshold established in
Section~\ref{sec:critical-threshold}: does the posterior concentrate on DAGs
that contain only a spurious edge between the confounded nodes while
preserving the correct causal relations among non-confounded variables, or on DAGs
that also corrupt those relations that were recovered correctly in the unconfounded
setting. Throughout this section, ``latent confounding'' refers to the
case when the partial correlation between the confounded variables induced by confounding is beyond the critical threshold of
Section~\ref{sec:critical-threshold}.

\subsection{Markov Equivalence and Posterior Connection}
\label{sec:orientation-expansion}
We now illustrate, through simple graphs, that posterior preference for DAGs over others is
governed not only by the BGe score, but also by the dynamics of the collider structures in the highest scoring DAGs.

Consider a chain $X_1 \to X_2 \to X_3$ as the underlying causal graph. It is
Markov equivalent to the reverse chain $X_1 \leftarrow X_2 \leftarrow X_3$
and the fork $X_1 \leftarrow X_2 \to X_3$ (where $X_2$ is the \emph{fork node},
the common cause of $X_1$ and $X_3$). These three DAGs are observationally
indistinguishable, and since the BGe score is score-equivalent within a MEC
\citep{geiger2002parameter}, Bayesian structure learning assigns them equal
posterior mass. The posterior probability of a directed edge is the total
mass of the DAGs containing it; the true edges $X_1 \to X_2$ and $X_2 \to X_3$
 receive only one-third and two-thirds of the posterior mass, respectively.

In contrast, the collider $X_1 \to X_2 \leftarrow X_3$ is uniquely
identifiable from observational data, since it induces a distinct
conditional independence pattern. Its MEC contains only a single DAG, so
the true edges are learned with high confidence. Now suppose $X_1$ and
$X_3$ are latently confounded. The score favours DAGs containing an edge
between $X_1$ and $X_3$ to explain the induced correlation. Once this edge
is present, the collider at $X_2$ is destroyed, and no
new collider can form anywhere among the three nodes since all three pairs of nodes are now adjacent. Consequently, every orientation of
the resulting fully connected three-node DAG, shown in
Figure~\ref{fig:triangle-and-parent}(a), is Markov equivalent and
receives the same score. Posterior mass therefore distributes across all
six equivalent DAGs in its MEC, and the true edges $X_1 \to X_2$ and $X_3 \to X_2$
receive substantially lower posterior probability than in the unconfounded
setting -- reducing confidence in their detection.

Now consider the true causal graph shown in Figure~\ref{fig:triangle-and-parent}(b). 
With no confounding, all DAGs Markov equivalent to it have the highest score, and since the MEC contains two DAGs, shown in  Figure~\ref{fig:triangle-and-parent}(b) and Figure~\ref{fig:triangle-and-parent}(c), the probability is distributed equally between the edges $X_1 \to X_4$ and $X_4 \to X_1$. Under confounding between $X_1$ and $X_3$, the DAG in Figure~\ref{fig:triangle-and-parent}(d) receives the highest score among all other DAGs due to parent asymmetry discussed in Section~\ref{sec:asymmetry-preferred}. However, in this case, though the collider at $X_2$ is destroyed, $X_4 \to X_1 \leftarrow X_3$ forms a new collider at $X_1$. The MEC now only contains this single graph and, hence the posterior probability of the true causal edge $X_4 \to X_1$
increases compared to the unconfounded setting.

\begin{figure}[htbp]
\centering
\begin{tabular}{cccc}
\begin{tikzpicture}[>=stealth, scale=0.5, every node/.style={font=\small}]
    \node (x1) {$X_1$};
    \node (x2) [above right of=x1] {$X_2$};
    \node (x3) [below right of=x2] {$X_3$};

    \draw[->] (x1) -- (x2);
    \draw[->] (x3) -- (x2);
    \draw[->] (x1) -- (x3);
\end{tikzpicture}
&
\begin{tikzpicture}[>=stealth, scale=0.5, every node/.style={font=\small}]
    \node (x4) {$X_4$};
    \node (x1) [right of=x4] {$X_1$};
    \node (x2) [above right of=x1] {$X_2$};
    \node (x3) [below right of=x2] {$X_3$};

    \draw[->] (x4) -- (x1);
    \draw[->] (x1) -- (x2);
    \draw[->] (x3) -- (x2);
\end{tikzpicture}
&
\begin{tikzpicture}[>=stealth, scale=0.5, every node/.style={font=\small}]
    \node (x4) {$X_4$};
    \node (x1) [right of=x4] {$X_1$};
    \node (x2) [above right of=x1] {$X_2$};
    \node (x3) [below right of=x2] {$X_3$};

    \draw[->] (x1) -- (x4);
    \draw[->] (x1) -- (x2);
    \draw[->] (x3) -- (x2);
\end{tikzpicture}
&
\begin{tikzpicture}[>=stealth, scale=0.5, every node/.style={font=\small}]
    \node (x4) {$X_4$};
    \node (x1) [right of=x4] {$X_1$};
    \node (x2) [above right of=x1] {$X_2$};
    \node (x3) [below right of=x2] {$X_3$};

    \draw[->] (x4) -- (x1);
    \draw[->] (x1) -- (x2);
    \draw[->] (x3) -- (x2);
    \draw[->] (x3) -- (x1);
\end{tikzpicture}

\\
(a) & (b) & (c) & (d) 
\end{tabular}
\caption{Four DAGs illustrating how collider configurations affect the size of the Markov equivalence class.}
\label{fig:triangle-and-parent}
\end{figure}
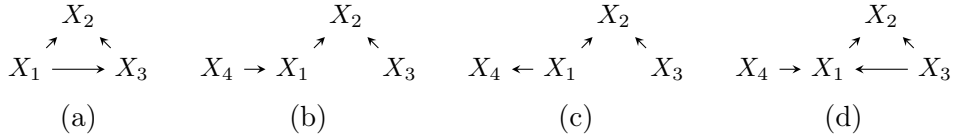
These examples demonstrate that size of the MEC of the highest scoring DAGs
determines the probability of recovering true edge directions. Importantly,
colliders are the essential graphical objects: their formation or destruction
directly changes the MEC of the graph, and consequently governs whether the
posterior remains concentrated or spreads across multiple competing
structures.

\subsection{Structural Characterization and Failure Regimes}
\label{sec:structural-characterization}

Consider a true causal DAG \(G=(V,E)\) over observed variables
\(\{X_1,\dots,X_n\}\), and suppose latent confounding is introduced
between variables \(X_j\) and \(X_k\), where \((X_j, X_k) \notin E\).
In the
absence of confounding, the highest-scoring structures are \(G\) and the
DAGs Markov equivalent to it, and posterior mass concentrates on this
MEC. The posterior probabilities of edges are determined by the frequency
with which those edges appear across these DAGs.

Under latent confounding, the covariance structure generated by the
original causal relations remains present in the data, so the true causal
directions continue to receive strong scores. However, the additional
dependence induced between \(X_j\) and \(X_k\) is not explained by \(G\).
Consequently, Bayesian structure learning assigns higher scores to
structures that additionally contain a spurious edge between the
confounded variables. The highest-scoring structures are therefore \(G\)
augmented with one of the two orientations of this edge, together with
their Markov equivalents:
\[
G \cup \{X_j \to X_k\}
\qquad\text{or}\qquad
G \cup \{X_k \to X_j\},
\]
and the DAGs Markov equivalent to each.

When both orientations satisfy acyclicity constraints, two factors
combine to determine the dominant orientation. First, one orientation
may achieve a higher score due to structural asymmetries such as
differing parent sets. Second, a smaller MEC concentrates posterior mass
on fewer DAGs. When both effects favour the same orientation it dominates
decisively; when they conflict, the outcome depends on whether the score
advantage outweighs the difference in MEC size.

Although MECs are properties of the full graph, latent confounding modifies
the observational structure only through the confounded variables and their
adjacent edges. We define the \emph{confounded region} as the local subgraph of the DAG \(G\) induced by the
confounded nodes together with their immediate neighbouring nodes in the underlying true causal graph. This does not contain the spurious edge. Consequently, relative to the unconfounded setting, changes in
the MEC arise from local changes in 
the confounded region. The remaining graph preserves the same equivalence
relations as before confounding, making the local confounded structure
sufficient for characterizing the resulting posterior behaviour even in larger
graphs.

Since MECs are determined by the skeleton and colliders, the key structural mechanism is whether the spurious edge preserves, creates, or weakens collider-based orientation constraints within the confounded region. This yields the
following \textit{structural characterisation} of the posterior behaviour, depending on whether the confounded region initially contains colliders. In what follows, 'number of colliders' refers to 
the number of distinct collider nodes, i.e. nodes with two or more non-adjacent
parents.  

\paragraph{Case 1: Confounded region contains colliders.}

\begin{itemize}
    \item[(a)] At least one orientation of the spurious edge preserves or
    increases the number of colliders in the confounded region. For e.g. in 
    Figure~\ref{fig:triangle-and-parent}(b) with confounding between $X_1$, $X_3$.
    \item[(b)] Both orientations of the spurious edge reduce the number of colliders in the confounded region. For e.g the collider $X_1 \to X_2 \leftarrow X_3$ with confounding between $X_1$, $X_3$.
\end{itemize}

\paragraph{Case 2: Confounded region contains no colliders.}

\begin{itemize}
    \item[(a)] At least one orientation of the spurious edge creates new colliders in the confounded region. For e.g. the chain  $X_1 \to X_2 \to X_3 \to X_4$ with confounding between $X_1$, $X_4$.
    
    \item[(b)] Neither orientation of the spurious edge creates
    colliders in the confounded region. For e.g. $X_1 \leftarrow X_2 \to X_3 $ with confounding between $X_1$, $X_3$.
\end{itemize}

These structural cases now induce two qualitatively different failure regimes,
where ``failure'' refers to the inference of a confounding-induced incorrect causal graph. The regimes differ in how this incorrect structure manifests in the
posterior:

\paragraph{Silent failure:}
The posterior assigns high confidence to both the spurious edge and the
surrounding true causal edges, making the inferred structure appear reliable
despite containing a spurious edge indistinguishable from true ones. Cases
1(a) and 2(a) correspond to this regime, where the small induced MEC keeps
posterior mass concentrated.

\paragraph{Noisy failure:}
The posterior assigns low confidence to surrounding true causal edges, visibly
signalling structural uncertainty in the confounded region. Cases 1(b) and 2(b) correspond to this regime: both produce a large MEC
that spreads posterior mass.
In Case 1(b), confounding actively causes this by destroying colliders.
In Case 2(b), the large MEC was already present due to absence of colliders,
and confounding leaves it unchanged---but the spurious edge is still
introduced, leaving true causal edges at low confidence.\\

\noindent
These regimes show that the effect of latent confounding is structurally
localised: depending on where it occurs in the true causal graph,
confounding may either introduce only a spurious edge while preserving
surrounding causal relations, or yield a posterior that visibly
loses confidence in the causal edges of the confounded region. A practitioner has no access to the unconfounded baseline and cannot tell
from the posterior alone whether a failure has occurred. Silent failures
are therefore particularly dangerous---the posterior appears confident
with no signal that the inferred structure contains a spurious edge.
Noisy failures, by contrast, are self-announcing through visibly low
confidence and appropriately warn the practitioner not to trust the
inferred causal relations.

\section{Experiments}
\label{sec:experiments}

We demonstrate the structural characterization of
Section~\ref{sec:structural-characterization} on synthetic
graphs of \(n = 5\) nodes. Small graphs admit exact posterior computation
via DAG enumeration, so any observed posterior behaviour is attributable
to confounding rather than inference error; and since the analysed
mechanisms are fundamentally local, small graphs suffice to isolate the
structural regimes.

\subsection{Data Generation and Evaluation Metrics}
\label{sec:exp-setup}

We sample 20 distinct (DAG, confounded pair) instances for each of the four structural case identified above at \(n=5\).
DAGs are generated by drawing a random topological order over the variables
and including each forward edge independently with probability \(p = 0.5\),
retaining only graphs whose edge count falls between 0.3 and 0.7 of all possible edges to avoid degenerate sparsity or near-complete graphs (which would
leave no non-adjacent pairs). For each DAG, every non-adjacent pair is classified into one of
the four cases by its collider structure. From each DAG we retain one
(DAG, pair) instance per case, and continue sampling DAGs until 20 such
instances are collected for every case.

Given a (DAG \(G\), confounded pair \((X_k, X_j)\)) instance, we generate
\(N = 5000\) i.i.d. samples from a linear Gaussian SCM with edge weights
\(w_{ij} \sim \mathcal{N}(0,2)\) constrained to \(|w_{ij}| \geq 1\),
observation noise \(\varepsilon_j \sim \mathcal{N}(0,1)\), and a latent
confounder \(H \sim \mathcal{N}(0,1)\) acting on the confounded pair:
\[
X_j = \sum_{i \in \mathrm{pa}(j)} w_{ij} X_i + \alpha H + \varepsilon_j,
\qquad
X_k = \sum_{i \in \mathrm{pa}(k)} w_{ik} X_i + \alpha H + \varepsilon_k.
\]
For each instance we draw \(10\) independent datasets by resampling the
edge weights and noise, giving 200 datasets per case in total.

For each DAG we report three metrics: \textbf{posterior entropy} \(H(P) = -\sum_{G'}
P(G' \mid \mathcal{D}) \log P(G' \mid \mathcal{D})\); the\textbf{ mean marginal
probability of the true edges of \(G\)}; and the \textbf{marginal probability of
the dominant orientation of the spurious edge}. The marginal probability
of an edge \(X_i \to X_j\) is the total posterior mass over all DAGs
containing that edge,
\(p(X_i \to X_j \mid \mathcal{D}) = \sum_{G' : X_i \to X_j \in G'}
P(G' \mid \mathcal{D})\).
We evaluate metrics at \(\alpha = 0\) (unconfounded) and \(\alpha = 0.8\).
Empirically, \(\alpha = 0.8\) induces a partial correlation exceeding the
critical value \(\rho_c\) (at \(N = 5000\)) for every sampled (DAG,
confounded pair) instance, ensuring all instances are past the threshold. Within each DAG, metrics are averaged across the 10 dataset runs and
then aggregated across the 20 DAGs per case as mean \(\pm\) standard
deviation. The true and spurious edge probabilities together capture the
practitioner-visible posterior state distinguishing silent from noisy
failure.
We use a uniform prior over DAG structures, \(P(G) \propto 1\), and
compute the exact posterior using the BGe score
\citep{geiger2002parameter}, evaluated via the corrected formulation of
\citet{kuipers2014addendum} with standard hyperparameter settings \(\alpha_\mu = 1\),
\(\alpha_w = n+2\), \(\nu = \mathbf{0}\), \(T = I_n\).

\subsection{Results}
\begin{table}[h]
\centering
\small
\begin{tabular}{lcccccc}
\toprule
& \multicolumn{2}{c}{Entropy} & \multicolumn{2}{c}{True edge prob} & \multicolumn{2}{c}{Spurious edge prob} \\
\cmidrule(lr){2-3}\cmidrule(lr){4-5}\cmidrule(lr){6-7}
Case & $\alpha{=}0$ & $\alpha = 0.8$ & $\alpha{=}0$ & $\alpha = 0.8$ & $\alpha{=}0$ & $\alpha = 0.8$ \\
\midrule
1(a) & 1.72 $\pm$ 0.45 & 1.67 $\pm$ 0.50 & 0.85 $\pm$ 0.08 & 0.87 $\pm$ 0.07 & 0.02 $\pm$ 0.02 & 0.83 $\pm$ 0.17 \\
1(b) & 1.72 $\pm$ 0.37 & 2.90 $\pm$ 0.43 & 0.85 $\pm$ 0.08 & 0.62 $\pm$ 0.09 & 0.08 $\pm$ 0.05 & 0.61 $\pm$ 0.07 \\
2(a) & 2.81 $\pm$ 0.30 & 2.29 $\pm$ 0.55 & 0.54 $\pm$ 0.08 & 0.70 $\pm$ 0.10 & 0.02 $\pm$ 0.01 & 0.89 $\pm$ 0.13 \\
2(b) & 2.86 $\pm$ 0.29 & 2.80 $\pm$ 0.35 & 0.57 $\pm$ 0.08 & 0.59 $\pm$ 0.10 & 0.03 $\pm$ 0.02 & 0.69 $\pm$ 0.12 \\
\bottomrule
\end{tabular}
\caption{Comparison between unconfounded (\(\alpha=0\)) and confounding 
beyond-threshold (\(\alpha=0.8\)) regimes of posterior metrics.}
\label{tab:aggregate}
\end{table}

Table~\ref{tab:aggregate} confirms the introduced case taxonomy. Cases 1(a) and 2(a)
keep the true edges highly confident while the spurious edge rises to
comparable confidence (silent failure), with 2(a) even improving the true
edges. Case 1(b) shows the noisy regime -- entropy nearly doubles and true
edge confidence collapses -- while Case 2(b) stays at low true edge
confidence with high entropy throughout. In all cases, the spurious edge,
negligible without confounding, appears at substantial probability under
confounding.

We now illustrate the transition across the confounding values 
\(\alpha \in [0,1]\) on a (DAG, confounded pair) instance
belonging to Case 1(a) (silent regime), shown in
Figure~\ref{fig:s2}.
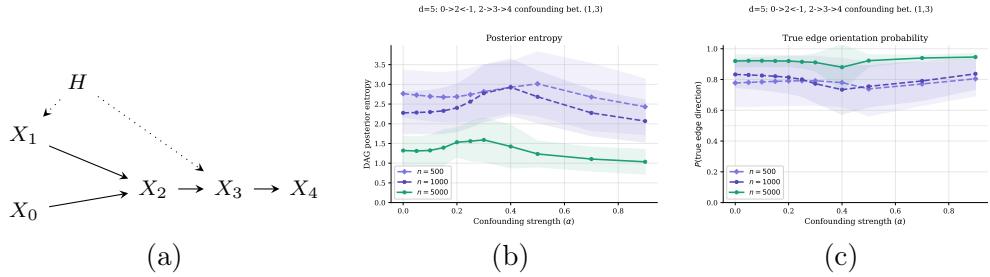
\begin{figure}[htbp]
\centering
\begin{tabular}{ccc}
\begin{tikzpicture}[>=stealth, scale=0.45, every node/.style={font=\scriptsize}]
    \node (x1) {$X_1$};
    \node (x0) [below of=x1] {$X_0$};
    \node (x2) [right of=x1, below right of=x1] {$X_2$};
    \node (x3) [right of=x2] {$X_3$};
    \node (x4) [right of=x3] {$X_4$};
    \node (h) [above right of=x1] {$H$};
    \draw[->] (x0) -- (x2);
    \draw[->] (x1) -- (x2);
    \draw[->] (x2) -- (x3);
    \draw[->] (x3) -- (x4);
    \draw[->, dotted] (h) -- (x1);
    \draw[->, dotted] (h) -- (x3);
\end{tikzpicture}
&
\includegraphics[height=3cm]{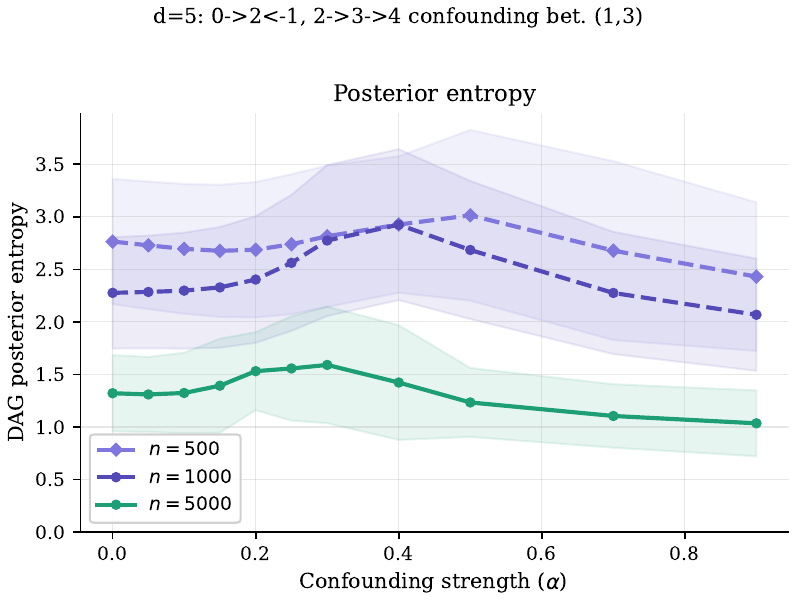}
&
\includegraphics[height=3cm]{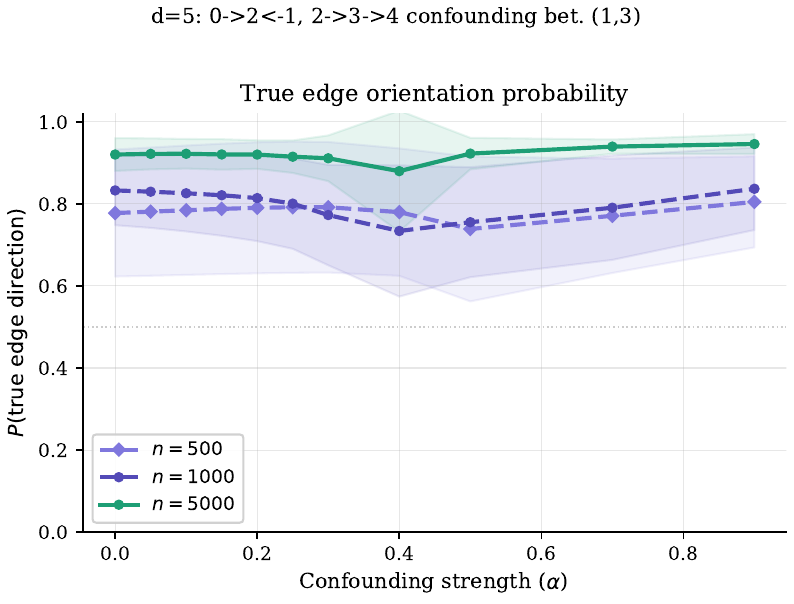}
\\
(a) & (b) & (c) 
\end{tabular}
\caption{(a) DAG with latent confounder \(H\); (b) posterior entropy; (c) true edge posterior probabilities.}
\label{fig:s2}
\end{figure}
Here, the
Posterior entropy (b) rises as competing DAGs gain score, peaks near the
critical threshold where structures with and without the spurious edge
explain the data equally well, then declines as the posterior
concentrates on the new small MEC. True edge probabilities (c) follow a
complementary trajectory: they remain stable or briefly decline near the threshold and recover beyond it, consistent with the silent regime. Increasing the sample size
$N$ shifts the transition to lower $\alpha$, illustrating the
sample-size dependence of the critical threshold derived in
Section~\ref{sec:critical-threshold}: more observations make the
spurious edge preferred at smaller partial correlation induced by weaker confounding coefficient $\alpha$.

\section{Discussion and Future Research}\textbf{}
\label{sec:discussion}
In this paper, we derive a critical correlation threshold above which the
BGe score favours a spurious edge between latently confounded variables.
We show that this threshold \emph{shrinks} with sample size -- meaning
that, paradoxically, more observational data makes the posterior more
susceptible to confounding, with even weakly induced correlations enough
to shift posterior mass toward incorrect causal structures. Beyond the threshold, the impact of confounding
is not uniform: it is shaped by how the induced spurious edge interacts
with the surrounding collider structure, producing
either broad posterior uncertainty (noisy failure) or highly confident but
incorrect structures (silent failure). Even in the simple setting considered
here, these failure modes reveal striking patterns in posterior behaviour
that practitioners would not anticipate from standard reliability cues. The
silent regime is particularly concerning, since highly confident posterior
edges may still correspond to structurally incorrect causal explanations. In such settings, even partial
domain knowledge about potentially confounded variable pairs can be valuable
for interpreting posterior confidence more cautiously and motivating additional
validation through interventions or sensitivity analysis.

Our analysis is restricted to linear Gaussian SCMs with a single additive
Gaussian latent confounder affecting two observed variables. Future work includes extending the characterization to non-Gaussian or
discrete settings, multiple interacting confounders, and larger confounded
regions involving several variables simultaneously. Another important direction
is understanding whether the same structural failure regimes remain observable
under practical approximate inference methods on much larger size graphs.
More broadly, these results  motivate confounder-aware posterior
diagnostics and uncertainty measures for Bayesian causal discovery.

\bibliographystyle{abbrvnat} 
\bibliography{references}

\appendix

\section{Derivation of the BGe Confounding Threshold}
\label{app:threshold-proof}

We prove Proposition~\ref{prop:local-critical-threshold}. The argument uses three standard
facts: decomposability of the BGe score, the large-sample BIC expansion of
Gaussian marginal likelihoods, and the expression of the Gaussian regression
likelihood ratio in terms of partial correlation.\\

\begin{proof}[Proof of Proposition~\ref{prop:local-critical-threshold}]
Let \(P = \operatorname{pa}_G(j)\) denote the current parent set of node
\(X_j\), and consider two DAGs \(G\) and \(G'\) that differ only by the
candidate edge \(X_k \to X_j\), so that
\(\operatorname{pa}_{G'}(j) = P \cup \{k\}\) and all other parent sets
coincide.

By decomposability in~\eqref{eq:decomp},
\(s(G) = \sum_{i=1}^{n} s(i, \operatorname{pa}_G(i))\). Since \(G\) and \(G'\) agree on every node
except \(X_j\), all node-wise terms cancel except that of \(X_j\), giving
\[
    s(G') - s(G)
    =
    s(j, P \cup \{k\}) - s(j, P)
    =: \Delta.
\]
The edge is favoured by the score precisely when \(\Delta > 0\).
For fixed hyperparameters and parent-set size, the BGe marginal likelihood
admits the Laplace (BIC) expansion
\citep{schwarz1978estimating, kass1995bayes}
\[
    s(j, P)
    =
    \ell(\hat\theta_P)
    - \frac{m_P}{2}\log N
    + O_p(1),
\]
where \(N\) is the number of i.i.d.\ samples, \(\ell(\hat\theta_P)\) is
the maximised Gaussian log-likelihood for the regression of \(X_j\) on
its parents \(X_P\), and \(m_P\) is the number of local parameters for
this regression. Adding \(X_k\) as a parent increases the local dimension
by one, \(m_{P \cup \{k\}} - m_P = 1\), so
\[
    \Delta
    =
    \ell(\hat\theta_{P\cup\{k\}}) - \ell(\hat\theta_P)
    - \frac{1}{2}\log N
    + O_p(1).
\]

For Gaussian linear regression, the maximised log-likelihood difference
equals
\[
    \ell(\hat\theta_{P\cup\{k\}}) - \ell(\hat\theta_P)
    =
    \frac{N}{2}\log\!\left(\frac{\mathrm{RSS}_P}{\mathrm{RSS}_{P\cup\{k\}}}\right),
\]
where \(\mathrm{RSS}_P\) is the residual sum of squares from regressing
\(X_j\) on \(X_P\). By the standard partial-correlation identity,
\(\mathrm{RSS}_{P \cup \{k\}} / \mathrm{RSS}_P = 1 - \hat\rho_{jk \mid P}^2\),
where \(\hat\rho_{jk \mid P}\) is the sample partial correlation between
\(X_j\) and \(X_k\) given \(X_P\). Hence
\[
    \Delta
    =
    \frac{N}{2}\log\!\left(\frac{1}{1 - \hat\rho_{jk \mid P}^2}\right)
    - \frac{1}{2}\log N
    + O_p(1).
\]

The first term grows linearly in \(N\) while the penalty grows as
\(\log N\), so for large \(N\) the sign of \(\Delta\) is governed by
these two terms. The score favours the enlarged parent set (\(\Delta > 0\))
when
\[
    N \log\!\left(\frac{1}{1 - \hat\rho_{jk \mid P}^2}\right) > \log N.
\]

Under the regular linear Gaussian data-generating process, the sample
partial correlation is consistent, \(\hat\rho_{jk \mid P} \xrightarrow{p}
\rho_{jk \mid P}\) as \(N \to \infty\). The population analogue of the
above condition defines a critical partial correlation \(\rho_c\) by
\[
    N \log\!\left(\frac{1}{1 - \rho_c^2}\right) = \log N,
    \qquad\text{i.e.}\qquad
    \rho_c^2 = 1 - N^{-1/N}.
\]
The score favours adding the edge whenever the confounding induces a
partial correlation exceeding \(\rho_c\).
\end{proof}

-

\end{document}